\documentclass{article}

    \usepackage[preprint]{neurips_2025}

\usepackage[utf8]{inputenc} 
\usepackage[T1]{fontenc}    
\usepackage{hyperref}       
\usepackage{url}            
\usepackage{booktabs}       
\usepackage{amsfonts}       
\usepackage{nicefrac}       
\usepackage{microtype}      
\usepackage{amsmath,amssymb}
\newcommand{\real}{\mathbb{R}}
\def \Xb{\mathbf{X}}
\def \Wb{\mathbf{W}}
\def \Hb{\mathbf{H}}

\usepackage{algorithm} 
\usepackage{algpseudocode} 
\usepackage{mathtools}
\usepackage{natbib}
\usepackage[table]{xcolor}
\usepackage{multirow}
\usepackage[textsize=tiny]{todonotes}
\usepackage{wrapfig}
\usepackage{graphicx}

\usepackage{amsthm}
\theoremstyle{plain}
\newtheorem{theorem}{Theorem}[section]

\newtheorem{lemma}[theorem]{Lemma}

\theoremstyle{definition}

\theoremstyle{remark}

\title{Learning Causal Graphs at Scale: \\A Foundation Model Approach}

\author{%
  Naiyu Yin\\
  Department of Mathematics \\
  Lehigh University \\
  Bethlehem, PA 18018 \\
  \texttt{nay224@lehigh.edu} \\
  \And
  Tian Gao \\
  IBM Research \\
  Yorktown Heights, NY 10598 \\
  \texttt{tgao@us.ibm.com} \\
  \AND
  Yue Yu \\
  Department of Mathematics \\
  Lehigh University \\
  Bethlehem, PA 18018 \\
  \texttt{yuy214@lehigh.edu} \\
}

\begin{document}

\maketitle

\begin{abstract}

  Due to its human-interpretability and invariance properties, Directed Acyclic Graph (DAG) has been a foundational tool across various areas of AI research, leading to significant advancements. However, DAG learning remains highly challenging, due to its super-exponential growth in computational cost and identifiability issues, particularly in small-sample regimes. To address these two challenges, in this work we leverage the recent success of linear transformers and develop a foundation model approach for discovering multiple order-consistent DAGs across tasks. In particular, we propose Attention-DAG (ADAG), a novel attention-mechanism-based architecture for learning multiple linear Structural Equation Models (SEMs). ADAG learns the mapping from observed data to both graph structure and parameters via a nonlinear attention-based kernel, enabling efficient multi-task estimation of the underlying linear SEMs. By formulating the learning process across multiple tasks as a continuous optimization problem, the pre-trained ADAG model captures the common structural properties as a shared low-dimensional prior, thereby reducing the ill-posedness of downstream DAG learning tasks in small-sample regimes. 
  We evaluate our proposed approach on benchmark synthetic datasets and find that ADAG achieves substantial improvements in both DAG learning accuracy and zero-shot inference efficiency. To the best of our knowledge, this is the first practical approach for pre-training a foundation model specifically designed for DAG learning, representing a step toward more efficient and generalizable down-stream applications in causal discovery.
\end{abstract}


\section{Introduction}\label{sec:intro}


Causality plays a fundamental role in explaining the underlying mechanisms of systems in many scientific and decision-making domains, including computer science, medicine, biology, and economics~\citep{pearl2000models, sachs2005causal, lu2021invariant, subbaswamy2020development}. This has led to significant interest within the machine learning community in developing advanced methods for causal discovery. A common approach to modeling causal relationships is to identify causal models among a set of random variables in the form of Directed Acyclic Graphs (DAGs), which offer a compact, interpretable, and theoretically grounded representation of the underlying data-generating process. 
\begin{figure}[hpt]
    \centering
    \includegraphics[width=1\linewidth]{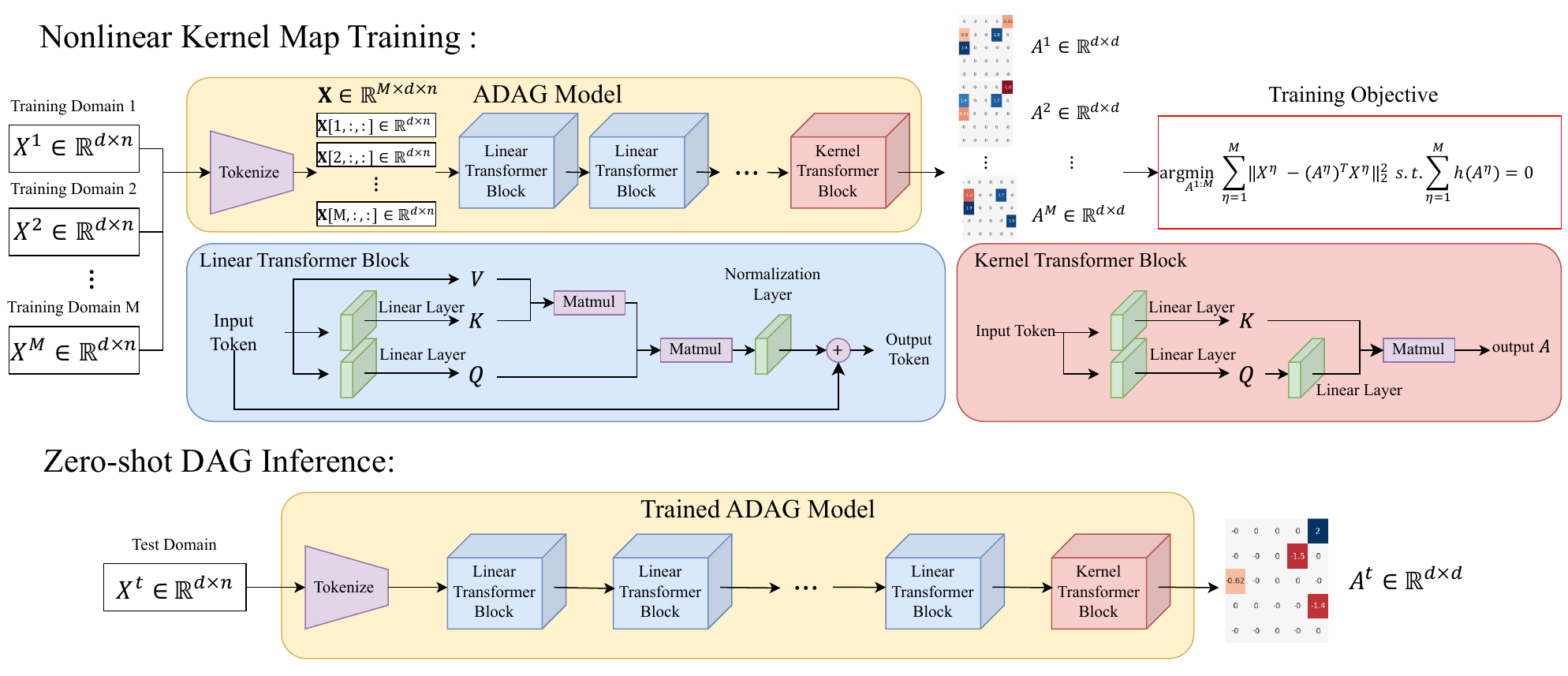}
    \caption{The schematic of the proposed ADAG training and inference procedures. Given order-consistent data from multiple domains, we begin by training a nonlinear kernel map from data to the corresponding hidden DAG. The trained model is then capable of zero-shot DAG inference, allowing it to predict weighted adjacency matrices when given data from new test domains.}
    \label{fig:pipeline}
\end{figure}
Numerous well-established methods have been developed for DAG learning, relying on either combinatorial search, score-based methods, or continuous relaxations. However, learning DAGs from observational data remains highly challenging due to the super-exponential space of possible graph structures, inherent identifiability issues, and data scarcity in real-world applications. Moreover, most existing DAG learning approaches operate on a per-task basis, lacking the ability to generalize across tasks or domains. As a result, there is growing interest in developing foundation models that can transfer knowledge across causal tasks, yet little progress has been made in this direction due to the unique structural constraints and identifiability challenges inherent to DAGs. 

Inspired by the recent success of foundation models and their capacity  to encode vast amounts of transferable information, we propose to address the high computational cost and poor performance in low-sample regimes by pre-training a DAG-learning foundation model that generalizes across tasks to infer order-consistent DAGs accurately and efficiently. Specifically, we introduce a novel attention-mechanism-based formulation for learning Structural Equation Models (SEMs). Although the proposed framework can be extended to more general models, for the purpose of demonstration, we focus primarily on the linear SEM setting in this work. 
As illustrated in Figure~\ref{fig:pipeline}, the key of our approach is to define a nonlinear kernel map using the attention blocks, which takes observational data as the input and the corresponding weighted adjacency matrices as the output. During the pre-training phase, the model is trained across multiple tasks, which each task treated as a DAG discovery problem (recovering a hidden DAG from a set of data observations). As such, the weighted adjacency matrix is inferred in an unsupervised way, and the model is capable to capture both the structural causal relationships and the associated causal mechanisms. 
The multi-task training and attention mechanism-based blocks were known to possess advantages in inferring shared low-dimensional structure across tasks and applying such a structure in the downstream test tasks \citep{lu2025transformer}. Therefore, our proposed ADAG architecture is anticipated to reduce the ill-posedness of individual DAG estimation problems and improve recovery of both graph structure and parameters, especially in data-scarce settings. Furthermore, once trained, our model enables zero-shot inference on previously unseen tasks, offering significant improvements in both accuracy and efficiency for downstream causal discovery. Notably, our framework does not require ground truth graphs, as the DAG kernel map is learned automatically using a data reconstruction loss, which differs from many supervised DAG learning approaches~\citep{li2020supervised}.

To the best of our knowledge, this is the first work to propose a practical approach for pre-training a foundation model specifically for DAG learning. Existing efforts to integrate causality with foundation models have very different focuses, such as causal inference~\citep{zhang2024towards} or semantic information extraction from existing large language models~\citep{ban2023query,wan2024bridging,wu2024causality}, as they do not directly train foundation models to estimate DAG structures and underlying causal model parameters from observed data alone.

While several prior works~\cite{chen2021multi, lu2023meta, zhang2017causal, zhou2022causal} have developed advanced algorithms for multi-task DAG learning, aiming to simultaneously estimate DAGs from either heterogeneous or order-consistent data. Those approaches primarily improve efficiency or accuracy through joint optimization across tasks. However, they do not generalize to unseen tasks without further optimization or training steps. In contrast, our approach supports zero-shot generalization to new tasks within heterogeneous and order-consistent data settings\footnote{In heterogeneous data settings, the data for new tasks are generated from the same DAG as the training data but with different causal mechanisms. In order-consistent data settings, the data for new tasks are generated from different DAGs and mechanisms, while sharing the same causal ordering as the training data.}, without requiring access to ground truth DAGs during training. This sets the stage for a new direction in causal discovery: building scalable, pre-trained foundation models capable of generalizing structural knowledge across domains.


\textbf{Major Contributions.} 1) We propose a novel attention-mechanism-based formulation for DAG learning, that learns a nonlinear kernel map from observational data to the underlying causal graph structure and associated parameters. 2) Using a pre-training procedure on multiple order-consistent DAG tasks, we obtain a foundation model for DAG learning. This model captures shared low-dimensional structures and enables zero-shot inference on unseen tasks within both heterogeneous and order-consistent data settings. 3) We empirically demonstrate that our method significantly outperforms both single-task and multi-task DAG learning baselines in terms of inference efficiency and accuracy, especially in small-sample regimes.


\section{Related Works}
\label{sec:related_work}


\textbf{Attention Mechanism for Inverse Problems.}
In recent years, the transformer based on the attention mechanism has been increasingly adopted to tackle diverse scientific problems. It has been found that the attention mechanism is capable of modeling complex dependencies within sequential or structured data, leading to novel applications in various domains \citep{guo2023Transformer, ovadia2024vito, yu2024nonlocal, evangelista2023ambiguity, chen2023deformable, cao2021choose}. Unlike many investigations on applying the attention mechanism for forward problems \citep{vladymyrov2024linear,lu2024asymptotic,zhang2024trained}, causal discovery and graph structure learning in general fall in the regime of inverse problems. In inverse problems, the objective is not merely to predict future outputs, but to infer the underlying (possibly linear) relationship that generates the observed data. To the authors' best knowledge, the attention mechanism and transformer in general have been relatively underexplored in the context of inverse problems, despite their potential to complement existing  deep learning strategies for inverse problems, such as  learning regularization hyperparameters~\citep{Afkham2021learning} or designing network architectures  explicitly for stable inversion~\citep{evangelista2025tobe}.

\textbf{Optimization-Based DAG Learning.} 
There has been extensive research on causal discovery and DAG learning, leading to a wide array of well-established algorithms. Traditional DAG learning approaches either quantify conditional independence relationships among variables through statistical tests or search for the optimal DAG by maximizing a predefined score using various search strategies. A notable shift was introduced by \cite{zheng2018dags}, who proposed reformulating the DAG learning problem from a combinatorial optimization task into a constrained continuous optimization problem, allowing for the use of gradient-based optimization methods. Subsequent works have improved various aspects of the continuous optimization framework, including enhancing the acyclicity constraint~\citep{yu2019dag}, optimization procedures~\citep{bello2022dagma, ng2020role}, efficiency and scalability~\citep{yu2021dags}, relaxing the strong assumptions of SEMs~\citep{khemakhem2021causal, yin2024effective}, and extending the framework to nonlinear SEM formulations~\citep{zheng2020learning, lachapelle2019gradient}. Despite these advances, DAG learning remains NP-hard~\citep{chickering2002optimal}, with the number of possible DAGs growing super-exponentially with the number of variables. Although the continuous optimization framework improves tractability, it does not eliminate the high computational cost. Moreover, existing approaches require a sufficient amount of data that accurately captures all underlying causal dependencies. As a result, their performance degrades when data is scarce. This limitation hinders their applicability in real-world scenarios, where data is often limited due to infrequent occurrences or high data collection costs. These two challenges motivate us to develop pre-trained models that are expressive enough to encode rich and transferable representations from available training data and generalize to unseen data, enabling efficient and accurate DAG inference even in small-data regimes.

\textbf{Multi-Task DAG Learning.} Since our training procedure involves recovering the underlying mechanisms between data observations and DAGs by jointly performing DAG learning across multiple training domains, the problem naturally falls within the multi-task learning setting. We therefore review existing works on multi-task DAG learning. In particular, \citet{chen2021multi} assumes that data from different tasks are generated by distinct DAGs that share a common topological order. They propose a DAG learning approach that simultaneously estimates task-specific DAGs by introducing a regularization term enforcing consistent topological ordering across tasks. In contrast, \citet{lu2023meta} and \citet{zhang2017causal} assume that the underlying DAG structure is shared across tasks, while the data generation mechanisms (i.e., causal mechanisms) vary. \citet{lu2023meta} formulates this problem as a constrained bilevel optimization and solves it using a bilevel primal-dual algorithm. \citet{zhang2017causal} focuses on time-varying causal mechanisms and proposes a constraint-based method to identify variables with changing causal relationships. While these approaches have shown promising results in multi-task DAG learning, they do not explicitly learn the mapping from data observations to DAG structures. As a result, the learned causal relationships or mechanisms may not generalize to unseen data observations that encode different causal structures or mechanisms from those in the training data. 


\section{Mathematical Formulation}
\label{sec:attention4DAG}


\subsection{Linear Structural Equation Model for Multi-Domain Data}

\vspace{-0.1in}

We begin by introducing our proposed attention-mechanism-based formulation for DAG learning. Given a set of $d$ random variables $X = [X_1, X_2, \cdots, X_d] \in \real^d$, the linear Structural Equation Model (SEM) with additive noise is defined as:
\begin{equation}
    X = A^TX + E
\end{equation}
where $A \in \mathbb{R}^{d \times d}$ is the weighted adjacency matrix representing the DAG. The entries of $A$ encode both the causal structure and the causal mechanisms, such that a nonzero entry \( A[i, j] \neq 0 \) indicates a causal link \( X_i \rightarrow X_j \).  The noise vector $E = [E_1, E_2, \cdots, E_d] \in \mathbb{R}^d$ consists of mutually independent exogenous noise variables.

In our setting, we assume the availability of $M$ domains of observations over the same set of $d$ variables, denoted as $\mathcal{D} = \{X_{1:d}^{\eta}\}_{\eta=1}^{M}$. The corresponding SEM for the $\eta$-th domain is
\begin{equation}\label{eq:sem}
    X_{1:d}^{\eta} = ( A^{\eta})^TX_{1:d}^{\eta} + E
\end{equation}
where $A^{\eta} \in \mathbb{R}^{d \times d}$ denotes the adjacency matrix of the DAG in the $\eta^{\text{th}}$ domain. for each domain \(\eta\), \(n\) observations of \(X_{1:d}^{\eta}\) are collected. We denote the collected data in domain \(\eta\) as \(\{X^{\eta}_{1:d}(j)\}_{j=1}^n\). Our goal is to infer the corresponding $A^\eta$ from the data $\{X^{\eta}_{1:d}(j)\}_{j=1}^n$ on the $\eta-$th domain. 

In this work, we focus on two relatively simple multi-domain causal discovery settings that exhibit certain shared structural consistencies across domains: (1) heterogeneous data, where all domains share the same DAG structure but differ in their domain-specific causal mechanisms, and (2) order-consistent data, where domains are governed by distinct DAGs and mechanisms but preserve a shared topological ordering. The key is to design an attention mechanism-based architecture and pre-train a foundation model. Naturally, a ultimate objective would be to obtain a foundation model for general graphs. However, we argue that the current two settings highlight the model capability in discovering structural consistency, and represents a first step toward the vision of foundation model pre-training for DAG learning. We leave the extension to fully general graph structures as an important direction for future research.


\subsection{Nonlinear Map from Data to Graph Structure and Parameters}


To recover the domain-specific adjacency matrix $A^{\eta}$ from the observed data $X_{1:d}^{\eta}$, we propose leveraging the expressive power of attention mechanisms to learn the underlying mapping between the observed data and the corresponding causal structure. To this end, we first model the weighted adjacency matrix $A^{\eta}$ as a function of the data $X_{1:d}^{\eta}$. 
Given the collected data in domain \(\eta\), \(\{X^{\eta}_{1:d}(j)\}_{j=1}^n\), we first transfer it to tokens \(\Xb^{\eta}(1:n)\):
\begin{equation}
    \Xb^{\eta}(1:n) = \Big(\Xb^{\eta}(1); \Xb^{\eta}(2); \cdots; \Xb^{\eta}(n)\Big) = \Big(X_{1:d}^{\eta}(1);X_{1:d}^{\eta}(2); \cdots; X_{1:d}^{\eta}(n)\Big)\in \real^{d \times n}.
\end{equation}
Here, each token represents the data from a variable, and it consists a vector of size $n$, concatenating the information from all $n$ samples on this domain. The weighted adjacency matrix \(A^{\eta}\) can then be modeled as a function of the tokens \(\Xb^{\eta}(1:n)\), dependent on trainable parameters \(\Theta\):
\begin{equation}
    A^{\eta} = A[\Xb^{\eta}(1:n);\Theta].
\end{equation}
We point out that although a linear SEM is considered in this work, the kernel map from the input $\Xb^{\eta}(1:n) \in \real^{d\times n}$ to the output, the weighted adjacency matrix $A^{\eta} \in \real^{d\times d}$, is highly nonlinear. To capture this complex nonlinear relation, we parameterize the function \(A^{\eta}[\cdot;\Theta]\) by designing an \(L\)-layer attention model:
\begin{equation}
    \begin{split}
        \Hb^{\eta}_{\text{in}} =&  \Hb^{\eta}_{(0)} := \Xb^{\eta}(1:n) \in \real^{d\times n}, \\
        \Hb^{\eta}_{(l)} :=& \text{Attn}[\Hb^{\eta}_{(l-1)}; \theta_l]\Hb^{\eta}_{(l-1)} + \Hb^{\eta}_{(l-1)}\in \real^{d\times n}, 1\leq l\leq L, \\
        A^{\eta} :=& \text{Attn}[\Hb^{\eta}_{(L)};\theta_{\text{out}}]  \in \real^{d\times d}, 
    \end{split}
\end{equation}
where the attention block writes:
\begin{equation}
    \text{Attn}[\Hb^{\eta}_{(l-1)}; \theta_l] = \sigma \Big(\frac{1}{\sqrt{d}}\Hb^{\eta}_{(l-1)}\Wb_l^Q(\Wb_l^K)^T(\Hb^{\eta}_{(l-1)})^T\Big) \in \real^{d\times d}.
\end{equation}
In the $l^{\text{th}}$ attention block, the trainable parameters are \(\theta_l = \{\Wb_l^Q \in \real^{n \times k}, \Wb_l^K \in \real^{n \times k}\}\), and \(\sigma(\cdot)\) is the activation function\footnote{ $\sigma(\cdot)$ is set to be the identity activation function in the paper, because it enables a more efficient implementation using linear attention \citep{liu2025neural}. However, other activation functions can also be used.}. In the last layer, we output the weighted adjacency matrix as:
\begin{equation}
    A^{\eta} = \text{Attn}[\Hb^{\eta}_{(L)};\theta_{\text{out}}] = \Wb^{P,x}_{\text{out}}\sigma \Big(\frac{1}{\sqrt{d}}\Hb^{\eta}_{(L)} \Wb_{\text{out}}^Q(\Wb_{\text{out}}^K)^T(\Hb^{\eta}_{(L)})^T\Big),
\end{equation}
where the trainable parameters are \(\theta_{\text{out}} = \{\Wb^{P, x}_{\text{out}}\in \real^{d\times d}, \Wb_{\text{out}}^Q \in \real^{n\times k}, \Wb_{\text{out}}^K\in \real^{n\times k}\}\). By substituting the above formulation into the SEM in Eq.~\eqref{eq:sem}, we have:
\begin{equation}\label{eq:sem_with_kernel_map}
    X^{\eta}_{1:d}(1:n) = A^T[X^{\eta}_{1:d}(1:n) ;\Theta] X^{\eta}_{1:d}(1:n)   + E(1:n),
\end{equation}
with $\Theta=\{\theta_l\}_{l=1}^L\cup \theta_{\text{out}}$.


\subsection{Attention Mechanism-Based DAG Learning}\label{sec:dag_learning}


Similar to prior continuous optimization-based DAG learning methods, we propose to learn a nonlinear kernel map by solving the following optimization problem:
\begin{equation}\label{eq:dag_learning}
    \begin{split}
        & \min_{\Theta := \{\theta_{\text{out}}, \theta_{1:L}\}} \sum_{\eta=1}^M \|X^{\eta}_{1:d}(1:n) - A^T[X_{1:d}^{\eta}(1:n);\Theta]X^{\eta}_{1:d}(1:n)\|_F^2 \\
        & \text{ s.t. } h(A[X_{1:d}^{\eta}(1:n);\Theta]) = 0, \forall \eta \in \{1, 2, \cdots, M\}.
    \end{split}
\end{equation}
Here, $h(A[X_{1:d}^{\eta}(1:n); \Theta]) = h(A^{\eta}) = \text{tr}(e^{A^{\eta} \circ A^{\eta}}) - d = 0$ is the acyclicity constraint proposed in \citet{zheng2018dags}, which ensures that $A^{\eta}$ represents the weighted adjacency matrix of a DAG. Our method does not impose restrictions on the choice of acyclicity constraint; alternative formulations of the DAG constraint from \citet{bello2022dagma} and \citet{zhang2022truncated} can also be used.




As shown in Eq.~\eqref{eq:dag_learning}, our goal is to learn a nonlinear kernel map from data observations to weighted adjacency matrices by jointly performing DAG learning across multiple data domains. This makes the optimization problem in Eq.~\eqref{eq:dag_learning} substantially more challenging than in the single-task DAG learning setting. Although recent methods for single DAG learning~\citep{bello2022dagma,ng2020role,yu2021dags} have improved efficiency by avoiding time-consuming iterative optimization, they are either not directly applicable to our pre-training scenario or yield suboptimal performance. To ensure the accuracy of the weighted adjacency matrices used in learning the nonlinear kernel map, we adopt the augmented Lagrangian method during training and solve a sequence of optimization subproblems.
With a slight abuse of notation, we re-write the optimization problem in Eq.~\eqref{eq:dag_learning} as:
\begin{equation}\label{eq:optimization_problem}
    \begin{split}
        & \max_{\alpha \in \mathbb{R}} \min_{\Theta} \mathcal{L}_{\text{rec}}\big(X_{1:d}^{1:M}(1:n);A^{1:M}\big) + \frac{\rho}{2} \sum_{\eta=1}^M |h(A^{\eta})|^2 + \alpha \sum_{\eta=1}^M h(A^{\eta}), \\
        & \mathcal{L}_{\text{rec}}\big(X_{1:d}^{1:M}(1:n);A^{1:M}\big) \coloneqq \frac{1}{2n} \sum_{\eta=1}^M \sum_{j=1}^n \sum_{i=1}^d \Big(X_i^{\eta}(j) - A^T\big[X_{1:d}^{\eta}(1:n);\Theta\big][i, :]X_{1:d}^{\eta}(j)\Big)^2.
    \end{split}
\end{equation}
where we denote $A[X_{1:d}^{\eta}(1:n);\Theta]$ as $A^{\eta}$ for simplicity in Eq.~\eqref{eq:optimization_problem}.
For each time, we solve the following optimization with the updated $\alpha$ value:
\begin{equation}
    \begin{split}
        \Theta_{\alpha}^* =& \arg\min_{\Theta} \mathcal{L}_{\text{rec}}\big(X_{1:d}^{1:M}(1:n);A^{1:M}\big) + \frac{\rho}{2} \sum_{\eta=1}^M |h(A^{\eta})|^2 + \alpha \sum_{\eta=1}^M h(A^{\eta}),
    \end{split}
\end{equation}
then update $A^{1:M}$ and $\alpha$ with Eq.~\eqref{eq:update_alpha}.\:
\begin{equation}\label{eq:update_alpha}
    (A^{\eta})_{\alpha}^* \leftarrow A[X_{1:d}^{\eta}(1:n);\Theta_{\alpha}^*], \qquad \alpha \leftarrow \alpha + \rho \sum_{\eta=1}^M h\big((A^{\eta})^*_{\alpha}\big).
\end{equation}
We summarize the proposed algorithm in Algorithm~\ref{alg:sum_array}.

\begin{algorithm}[hpt]
\caption{Attention-DAG (ADAG) Training Process}\label{alg:sum_array}
\begin{algorithmic}[1] 
    \State \textbf{Input:} Training domain data $X_{1:d}^{1:M}(1:n)$, initial guesses of $\Theta_0$ and $\alpha_0$, progress rate $c\in (0, 1)$, tolerance $\epsilon >0$, threshold $\omega > 0$
    \For{$t \gets 1$ \textbf{to} $n$} 
        \State Solve $\Theta_{t+1} \leftarrow \arg\min_{\Theta} \mathcal{L}_{\text{rec}}\big(X_{1:d}^{1;M}(1:n);A^{1:M}\big) + \frac{\rho}{2} \sum_{\eta=1}^M |h(A^{\eta})|^2 + \alpha_t \sum_{\eta=1}^M h(A^{\eta})$,
        with $\rho$ such that $\sum_{\eta=1}^M h(A^{\eta}_{t+1})< c \sum_{\eta=1}^M h(A^{\eta}_t)$.
        \Comment{Use Adam optimizer}
        \State Update $A^{1:M}_{t+1} \leftarrow A[X_{1:d}^{1:M}(1:n);\Theta_{t+1}]$.
        \State Update $\alpha_{t+1} \leftarrow \alpha_t + \rho \sum_{\eta=1}^M h(A_{t+1}^\eta)$.
        \If{$\sum_{\eta=1}^M h(A_{t+1}^\eta) < \epsilon $}
        \State $\hat{\Theta}= \Theta_{t+1}$ and break.
        \EndIf
    \EndFor
    \State \Return the optimal parameters $\hat{\Theta}$ for the nonlinear kernel map $A[\cdot;\Theta]$.
\end{algorithmic}
\end{algorithm}

\textbf{Nonlinear SEM Extension.} Although we formulate our attention-mechanism-based DAG learning problem under the linear SEM assumption, the idea can be easily extended to the nonlinear SEM setting by first applying nonlinear transformations to the input $X_{1:d}^{\eta}(1:n)$ before multiplying it with weighted adjacency matrices. In fact, we can use shared attention layers for both the nonlinear kernel map and the nonlinear transformation, and multiply $A$ with $\Hb^{\eta}_{(L)}$ instead of $X_{1:d}^{\eta}(1:n)$ in Eq.~\eqref{eq:sem_with_kernel_map}.

\textbf{Discover the Prior to Enhance Identifiability. }While most single-task DAG learning methods mainly consider sufficient rank problems ($n\gg d$), in this work we consider a more challenging scenario with small observed data in each domain. In the latter case, the inverse problem may become under-determined, making the learning non-identifiable. As shown in \cite{lu2025transformer}, the linear transformer is capable of alleviating this issue, by implicitly discovering the low-dimensional shared structure from the training dataset of multiple domains and automatically applying it as prior information in downstream test tasks. Hence, we anticipate that the linear attention mechanism is capable of discovering the shared structural consistencies in our two causal discovery settings, so as to mitigate the deficiency rank issue in small observed data DAG learning problems. Further discussions on the identifiability of ADAG are provided in Appendix \ref{sec:theory}, where we divide the theoretical analysis into two parts: the recovery of a consistent order and the parameters in the weighted adjacency matrix $A$, and show that minimizing the joint data loss from all $M$ domains would help to enhance the identifiability in both parts. In our empirical experiments, we validate these prospectives by showing that: 1) the linear transformer finds a low-dimensional structure in the prior distribution where the ground-truth $A$ is drawn (see Figure \ref{fig-eigenvalues}); and 2) our ADAG is capable of recovering both the correct graph structure and the parameters in $A$, even with a relatively small $n$.


\section{Experiments}

In this section we evaluate the performance of our proposed ADAG algorithm, focusing on three key aspects: 1) The attention-mechanism-based kernel map learned by ADAG successfully captures the underlying mechanisms between data observations and their generative processes, which are encoded in a DAG. It generalizes well to unseen domains of data, provided that the data are generated from DAGs that respect either the topological structures or the orderings used during training. 2) We assess ADAG's zero-shot DAG inference performance in terms of both accuracy and efficiency across multiple settings. Empirically, our method consistently outperforms existing state-of-the-art single-task and multi-task DAG learning baselines. 3) We further show that, since the kernel map produced by ADAG captures the common low-dimensional structure across domains, it mitigates the ill-posedness of DAG learning in low-sample regimes, enabling more reliable DAG learning when training samples are limited. We present the main results that best support our claims in the main paper and provide additional results in Supplementary \ref{app:detailed_empirical_results}.

\textbf{Dataset Settings.} We follow standard protocols for generating synthetic graphs and data. The ground-truth DAGs are sampled from Erdős–Rényi (ER) graphs with a degree $k=1$\footnote{Graphs with degree $k$ have an expected number of edges equal to $kd$.}. We consider graphs of varying sizes with $d = \{5, 10, 20\}$ nodes. Data is then generated using a linear SEM, where the coefficients for each edge are drawn from $U[-2, -0.5] \cup U[0.5, 2]$. We consider two types of data: (1) Heterogeneous data, generated following the procedure in \citep{lu2023meta}. For training, we generate $M$ domains, each with $n$ samples, using the same DAG structure but with varying edge weights. An additional 1000 domains are held out for testing and are not seen during pre-training. (2) Order-consistent data, generated following the procedure in~\citep{chen2021multi}. Here, each domain contains $n$ samples generated from different DAGs that share the same topological ordering. We again generate $M$ domains for training and 1000 for zero-shot DAG inference. To illustrate the distinction between these two data types, we visualize their weighted adjacency matrices in Supplementary \ref{app:data_visualization}.

\textbf{Implementation Details.} All experiments are conducted on a single NVIDIA GeForce RTX 3090 GPU using the Adam optimizer for training. Detailed hyperparameter settings—such as the number of layers $nb$, the attention head dimension $k$, batch size, learning rate, total training epochs, and the initialization values for the augmented Lagrangian multipliers $\rho$ and $\lambda$—are provided in Supplementary \ref{app:implementation_details}. Each experiment is repeated for three times, and we report the average results.


\subsection{Generalization of the Nonlinear Kernel Map}


\begin{figure}[hpt]
    \centering
    \includegraphics[width=\textwidth]{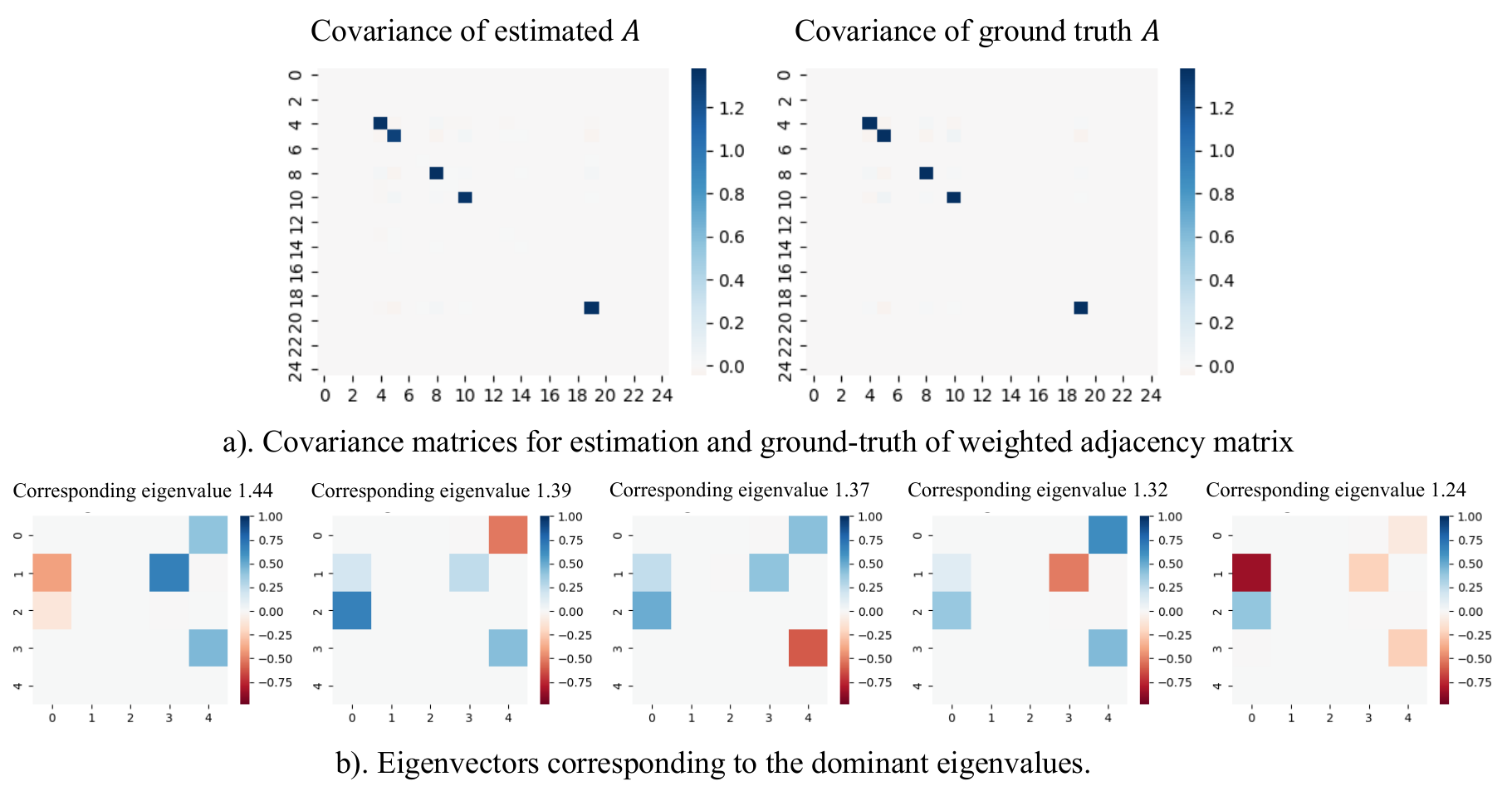}\vspace{-0.1in}
    \caption{Illustration of the learned kernel map on the order-consistent data generated from an ER1 graph with $d=5$. Figure (a) shows the covariance matrices of the estimated and ground-truth weighted adjacency matrices. Figure (b) shows the eigenvectors corresponding to the dominant eigenvalues. The eigenvectors are aligned with the ground-truth DAGs.} \label{fig-eigenvalues}
       \vskip -3mm
\end{figure}
In this section, we show that the attention-mechanism-based kernel map we learned by solving multiple DAG learning problems on training data can generalize well to unseen held-out domains if they are generated from DAGs with either the same topological structures or orderings as those used in the training domains. First, we illustrate that the learned kernel map can identify the common low-dimensional structure across domains. Specifically, we apply the trained kernel map to infer the weighted adjacency matrices $A^{\text{test},\eta} \in \mathbb{R}^{d \times d}$ for 1000 held-out test domains generated from order-consistent data with $d = 5$. As shown in Figure~\ref{fig-eigenvalues}(a), the covariance matrix of the estimated adjacency matrices closely matches that of the ground-truth matrices, indicating that the learned map preserves the underlying structural patterns. We further compute the eigenvalues and eigenvectors of the estimated adjacency matrices and plot those associated with the dominant eigenvalues in Figure~\ref{fig-eigenvalues}(b). As expected, there are five dominant eigenvalues across all test domains, consistent with the structure of an ER1 graph with $d = 5$. We also visualize the five corresponding eigenvectors in Figure~\ref{fig-eigenvalues}(b). These vectors align well with the five edges of the ground-truth DAG, supporting our claim that the learned kernel map accurately captures the common low-dimensional structure across domains with high fidelity.

Beyond its ability to identify the common low-dimensional structure, we also show that the learned kernel map is expressive enough to accurately predict the coefficients of the edges in the DAGs. As shown in Figure~\ref{fig-coefficients}, the weighted adjacency matrices predicted by the learned kernel map closely match the ground-truth weighted adjacency matrices.
\begin{figure}[hpt]
    \centering
    \includegraphics[width=\textwidth]{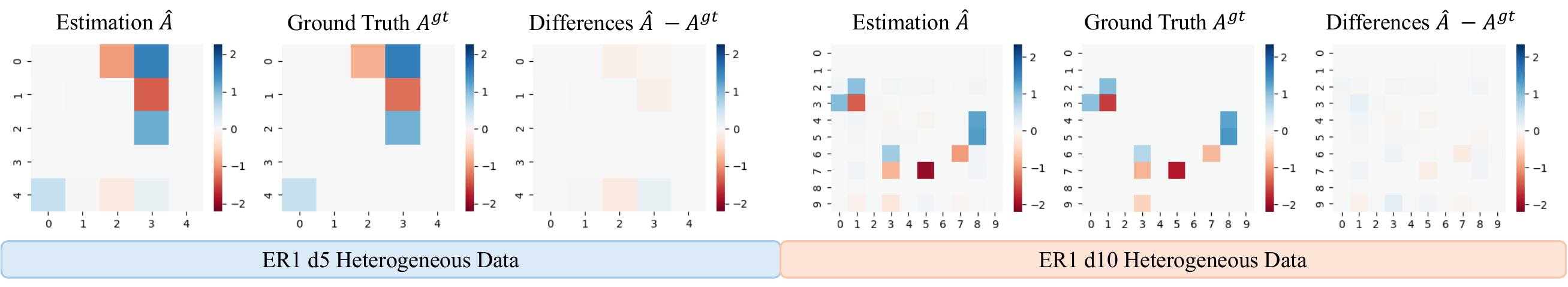}\vspace{-0.1in}
    \caption{Visualization of the estimation and ground-truth weighted adjacency matrices, along with their difference on the heterogeneous data generated from ER1 graph with $d=5$ and $d=10$.} \label{fig-coefficients}
\end{figure}

We observe that this generalization capability hinges on having a sufficient number of domains in the training set. To further investigate this, we conduct an ablation study by varying the number of training domains. As the number of domains increases, the performance of the learned kernel map improves, leading to lower input reconstruction error and reduced relative error between the estimated and ground-truth weighted adjacency matrices. Please refer to Supplementary \ref{app:ablation_M} for more details.



\subsection{Zero-Shot DAG Inference}


We evaluate the zero-shot DAG inference performance of our proposed ADAG algorithm. Specifically, we apply the learned kernel map—trained on the $M$ training domains—to infer the weighted adjacency matrices for 1000 held-out test domains. Following the common practice in DAG learning \citep{zheng2018dags}, a threshold is applied on the inferred weighted adjacency matrices, with a fixed value of 0.3 across all experiments. To assess accuracy, we use the Structural Hamming Distance (SHD), which counts the number of extra, missing, and reversed edges in the inferred DAGs relative to the ground-truth DAGs. 
To evaluate how well the edge coefficients are recovered, we compute the relative error between the estimated and ground-truth weighted adjacency matrices. Additionally, we report the runtime of the inference process to demonstrate the efficiency of our approach.

We compare the performance of ADAG with several state-of-the-art single-task and multi-task DAG learning baselines, including NOTEARS~\citep{zheng2018dags}, DAGMA~\citep{bello2022dagma}, CD-NOD~\citep{zhang2017causal}, MetaDAG~\citep{lu2023meta}, and MultiDAG~\citep{chen2021multi}. For the single-task methods NOTEARS and DAGMA, we apply each algorithm independently to every domain. For the multi-task methods MetaDAG and MultiDAG, we run the algorithms jointly on the 1000 test domains to simultaneously learn the domain-specific weighted adjacency matrices. For CD-NOD, which is a constraint-based method that returns a DAG skeleton rather than a weighted adjacency matrix, we concatenate all observations from the 1000 test domains and use the combined dataset as input. Since CD-NOD does not output edge weights and may return undirected edges, we exclude the edge count and relative error metrics for this baseline. To ensure a fair comparison across all methods, we either extensively tune the hyperparameters or adopt the recommended settings reported in the original papers.

Table~\ref{tab:DAG_inference_performance} presents the empirical results of all methods on both heterogeneous and order-consistent datasets. All results are averaged over the 1000 test domains and three independent runs. From Table~\ref{tab:DAG_inference_performance}, we observe that our proposed ADAG method, which leverages information learned from the training domains, achieves the best overall performance in terms of DAG inference accuracy (lowest SHD and relative error) and zero-shot inference efficiency, consistently outperforming all state-of-the-art baselines.

Moreover, we find that, when trained on a sufficient number of domains, the learned kernel map in ADAG accurately captures the underlying low-dimensional causal structure with high precision. This eliminates the need for additional sparsity regularization. As shown in the table, the number of edges in the inferred weighted adjacency matrices closely matches that of the ground-truth DAGs.

\begin{table}[hpt]
    \scriptsize
    \centering
    \caption{Evaluation of the zero-shot DAG inference performance on ER$1$ \textbf{heterogeneous} data and \textbf{order-consistent} data with varying scales.}
    \label{tab:DAG_inference_performance}
    \begin{tabular}{|c|c|cccc|cccc|}
        \hline 
        \multirow{2}{*}{d} & \multirow{2}{*}{Methods} &  \multicolumn{4}{c|}{Heterogeneous Data} & \multicolumn{4}{c|}{Order-consistent Data}\\
        \cline{3-10} 
        & &SHD$\downarrow$ & $\frac{|\hat{A} - A^{gt}|}{|A^{gt}|}\downarrow$  & $\#$ edges & runtime (s)$\downarrow$ &SHD$\downarrow$ & $\frac{|\hat{A} - A^{gt}|}{|A^{gt}|}\downarrow$  & $\#$ edges & runtime (s)$\downarrow$\\
        \hline 
        \multirow{6}{*}{5} & NOTEARS & 0.4320 & 0.1678 & 5.2950 & 0.0742 & .3420 & 0.1667 & 4.2840 & 0.1022 \\
         & DAGMA & 0.4130 & 0.1623 & 5.3310 & 0.3776 & 0.3650 & 0.1711 & 4.3050 & 0.3662 \\
         & MetaDAG & 5.7000 & 1.0105 & 4.7000 & 72.0971 & 5.7000 & 1.0087 & 5.0000 & 72.8796 \\
         & CD-NOD & 1.5300 & - & - & 2.6404 & 0.7300 & - & - & 2.2296 \\
         & MultiDAG &  0.0470 & 0.1531 & 5.0230 & 0.0151  & 0.0790 & \bf{0.1179} & 4.0630 & 1.9911 \\
         \rowcolor{gray!15}
         & ADAG & \bf{0.0270} & \bf{0.1242} & 4.9730 & \bf{0.0003} & \bf{0.0380} & 0.1186 & 4.0080 & \bf{0.0003} \\
         \hline 
         \multirow{6}{*}{10}& NOTEARS & 1.1280 & 0.2403 & 10.3930 & 0.3880 & 1.1000 & 0.2121 & 10.5970 & 0.4875 \\
         & DAGMA & 0.8690 & 0.1975 & 10.5390 & 0.4564 & 0.8690 & 0.1975 & 10.5390 & 0.7658\\
         & MetaDAG & 17.5000 & 1.0122 & 10.5000 & 213.9723 & 19.0000 & 1.0079 &12.0000  & 232.2737 \\
         & CD-NOD & 3.7000 & - & - & 16.5509 & 4.1200 & - &  -& 21.5743\\
         & MultiDAG & 0.5120 & 0.1884 & 10.3760 & 0.1622 & 0.1810 & 0.1415 & 10.1390 & 2.9333 \\
         \rowcolor{gray!15}
         & ADAG & \bf{0.0750} & \bf{0.1502} & 9.9250 & \bf{0.0004}  & \bf{0.1530} & \bf{0.1267} & 10.0850 & \bf{0.0003} \\
         \hline 
         \multirow{6}{*}{20} & NOTEARS & 3.5900 & 0.3053 & 22.2200 & 2.1522  & 2.6960 & 0.2567 & 22.1110 & 6.0447 \\
         & DAGMA & 2.8450 & 0.2724 & 22.0690  & 0.7000 & 2.8450 & 0.2724 & 22.0690 & 1.2426 \\
         & MetaDAG & 31.4000 & 1.0254 & 22.1000 & 456.5278 & 33.2000 & 1.0215 & 21.0000 & 498.2264 \\
         & CD-NOD & 11.7500 & - & - & 67.4182 & 7.9800 & - & - & 39.6735\\
         & MultiDAG & 0.7890 & 0.2477 & 20.3990 & 0.4530 & 0.6440 & 0.1649 & 21.0040 & 6.0640 \\
         \rowcolor{gray!15}
         & ADAG & \bf{0.2412} & \bf{0.2258} & 20.7000 & \bf{0.0003} & \bf{0.2170} & \bf{0.1490} & 20.0400 & \bf{0.0004} \\
         \hline 
    \end{tabular}
\end{table}

\textbf{Evaluation on real dataset. }Our ADAG approach, although effective, requires access to data from a sufficient number of domains. Therefore, it is not suitable to directly apply ADAG to benchmark real datasets where the number of domains is limited. Nevertheless, we design a procedure using the Sachs dataset~\cite{sachs2005causal}, which allows us to sample suitable synthetic data to serve as training domain data. The Sachs dataset contains real-world flow cytometry measurements for modeling protein signaling pathways and comprises 11 continuous variables with 853 observations. We apply our method to pre-train on the sampled training data and then perform DAG inference on the real data observations. Empirical results show that our method outperforms several state-of-the-art single-task baselines given limited observations. Please refer to Supplementary \ref{app:real_data_results} for the detailed data sampling procedure and numerical results.


\subsection{Low-Sample Robustness}


We further investigate the capability of our kernel map from ADAG in mitigating the ill-posedness of DAG learning in low-sample regimes. To evaluate this, we reduce the number of available observations per domain in both the training and test sets to $n = 25$ and $n = 50$ for the $d = 5$ setting. We then assess the performance of our method and all baselines under these limited-data scenarios and summarize the results in Table~\ref{tab:low_data_regime}. While other baselines experience severe performance degradation at $n = 25$ and $n = 50$, our method experiences a relatively modest drop in accuracy, demonstrating greater robustness in the low-sample regime. Additionally, it is of interests to see if our pre-trained foundation model can generalize to test tasks with small samples. To this end, we train ADGA with $n = 100$ observations per domain and apply it to infer the weighted adjacency matrices for test domains with only $25$ or $50$ observations per domain.
To perform the downstream tests, we randomly sample from the available $n=25$ or $50$  observations in each test domain, and augment them with duplicates until the total reaches $100$. 
Then, we use this augmented data as the input for inference in ADAG. As shown in Table~\ref{tab:low_data_regime}, models trained with larger $n$ values consistently improve DAG learning accuracy, though at the cost of slightly increased inference time. These results highlight a key advantage of our ADAG approach: its ability to generalize effectively from high-resource to low-resource settings. By pre-training on domains with sufficient data, the model can perform accurate DAG inference even when test domains have severely limited data, simply by leveraging augmentation strategies to align with the pre-training regime. Such flexibility and adaptability make our model particularly well-suited for real-world applications, where data scarcity is common and collecting additional observations can be costly or infeasible.

\begin{table}[hpt]
    \centering
    \scriptsize
    \caption{Evaluation of the zero-shot DAG inference performance on ER$1$ \textbf{heterogeneous} data and \textbf{order-consistent} data under low-samples regime.}
    \label{tab:low_data_regime}
    \begin{tabular}{|c|c|cccc|cccc|}
        \hline 
        \multirow{2}{*}{n} & \multirow{2}{*}{Methods} &  \multicolumn{4}{c|}{Heterogeneous Data} & \multicolumn{4}{c|}{Order-consistent Data}\\
        \cline{3-10} 
        & &SHD$\downarrow$ & $\frac{|\hat{A} - A^{gt}|}{|A^{gt}|}\downarrow$  & $\#$ edges & runtime (s)$\downarrow$ & SHD$\downarrow$ & $\frac{|\hat{A} - A^{gt}|}{|A^{gt}|}\downarrow$  & $\#$ edges & runtime (s)$\downarrow$\\
        \hline 
         \multirow{7}{*}{50}& NOTEARS & 0.6590 & 0.2139& 5.3850 & 0.0757 & 0.5400 & 0.2049 & 4.4020 & 0.0945\\
         & DAGMA & 0.6080 & 0.2028 & 5.4110 & 0.3716 & 0.5440 & 0.2061 & 4.4220 & 0.3591\\
         & MetaDAG & 5.8000 & 1.1018 & 4.7000 & 73.9802 & 5.7000 & 1.0088 & 5.0000 & 72.7758 \\
         & CD-NOD & 2.4000 & - & - & 0.7445 & 1.2900 & - & - & 0.6622 \\
         & MultiDAG & 0.1110 & 0.1751 & 5.0570 & 0.0114 & 0.5830 & 0.2688 & 4.5430 & 0.4502 \\
         \rowcolor{gray!15}
         & ADAG ($n=50$) & 0.0550 & 0.1573 & 4.9750 & \bf{0.0004} & 0.1880 & 0.1719 & 4.1320 & \bf{0.0003} \\
         \rowcolor{gray!15}
         & ADAG ($n=100$) & \bf{0.0540} & \bf{0.1528} & 4.9460 & 0.0006  & \bf{0.0780} & \bf{0.1470} & 4.0260 & 0.0004\\
         \hline 
         \multirow{7}{*}{25} & NOTEARS & 1.1870 & 0.2989 & 5.5710 & 0.0710 & 1.5110 & 0.3320 & 5.2160 & 0.0960 \\
         & DAGMA & 1.1220  & 0.2852 & 5.5880  & 0.3851 & 1.5590  & 0.3408 & 5.2540  & 0.3657 \\
         & MetaDAG & 5.7000 & 1.1016 & 4.7000 & 73.8623 & 5.7000 & 1.0089 & 4.9000 & 72.7697 \\
         & CD-NOD & 3.6000 & - & - & 0.3926 & 2.3400 & - & - & 0.3041 \\
         & MultiDAG & 0.5610 & 0.2458 & 5.3450& 0.0097 & 1.0700 & 0.3133 & 4.9180 & 0.3816\\
         \rowcolor{gray!15}
         & ADAG ($n=25$) & 0.2710 & 0.2295 & 5.0050  & \bf{0.0004} & 0.6840  & 0.2485 & 4.5200 & \bf{0.0003}\\
         \rowcolor{gray!15}
         & ADAG ($n=100$) & \bf{0.1250} & \bf{0.1955} & 4.8870 & 0.0006 & \bf{0.1920} & \bf{0.1886} & 4.0820 & 0.0004 \\
         \hline 
    \end{tabular}
\end{table}


\section{Conclusion}


In this paper, we propose ADAG, a novel attention mechanism-based approach for training a foundation model for DAG learning. The core of our method is a nonlinear kernel mapping that captures the relationship between data observations and their underlying causal structures and mechanisms. By jointly training the model with optimization-based DAG learning approach across multiple domains, ADAG is designed to generalize effectively to test domains with unseen DAGs and mechanisms. Empirically, we demonstrate that the learned kernel map accurately captures the common low-dimensional causal structure and predicts edge coefficients with high precision. Evaluations on benchmark synthetic datasets show that ADAG achieves significant improvements in both DAG learning accuracy and zero-shot inference efficiency. Furthermore, our model exhibits strong robustness in low-sample regimes. To the best of our knowledge, this is the first practical approach for pre-training a foundation model specifically for DAG learning, marking an important step toward the development of pre-trained models for causal discovery.

\textbf{Limitations and Broader Impact.} Due to computational resource limit, our experiments focus on learning from data generated by linear models with variable size $d \leq 20$. It would be beneficial to test the proposed method to larger variable sizes. Our work takes a meaningful first step toward building generalizable and data-efficient causal discovery systems by introducing a foundation model pre-trained for DAG learning. This has the potential to benefit domains where causal inference is critical but labeled or interventional data are scarce. 


\begin{ack}

NY was supported by the National Institute of Health award 1R01GM157589-01. YY was supported by the AFOSR grant FA9550-22-1-0197 and the National Science Foundation (NSF) award DMS-2427915. Portions of this research were conducted on Lehigh University's Research Computing infrastructure partially supported by NSF Award 2019035.


\end{ack}

\bibliographystyle{plainnat}   

\begin{thebibliography}{42}
\providecommand{\natexlab}[1]{#1}
\providecommand{\url}[1]{\texttt{#1}}
\expandafter\ifx\csname urlstyle\endcsname\relax
  \providecommand{\doi}[1]{doi: #1}\else
  \providecommand{\doi}{doi: \begingroup \urlstyle{rm}\Url}\fi

\bibitem[Afkham et~al.(2021)Afkham, Chung, and Chung]{Afkham2021learning}
Babak~Maboudi Afkham, Julianne Chung, and Matthias Chung.
\newblock Learning regularization parameters of inverse problems via deep neural networks.
\newblock \emph{Inverse Problems}, 37\penalty0 (10):\penalty0 105017, sep 2021.
\newblock \doi{10.1088/1361-6420/ac245d}.
\newblock URL \url{https://dx.doi.org/10.1088/1361-6420/ac245d}.

\bibitem[Ban et~al.(2023)Ban, Chen, Wang, and Chen]{ban2023query}
Taiyu Ban, Lyvzhou Chen, Xiangyu Wang, and Huanhuan Chen.
\newblock From query tools to causal architects: Harnessing large language models for advanced causal discovery from data.
\newblock \emph{arXiv preprint arXiv:2306.16902}, 2023.

\bibitem[Bello et~al.(2022)Bello, Aragam, and Ravikumar]{bello2022dagma}
Kevin Bello, Bryon Aragam, and Pradeep Ravikumar.
\newblock Dagma: Learning dags via m-matrices and a log-determinant acyclicity characterization.
\newblock \emph{Advances in Neural Information Processing Systems}, 35:\penalty0 8226--8239, 2022.

\bibitem[Cao(2021)]{cao2021choose}
Shuhao Cao.
\newblock Choose a transformer: {Fourier or Galerkin}.
\newblock \emph{Advances in neural information processing systems}, 34:\penalty0 24924--24940, 2021.

\bibitem[Chen et~al.(2023)Chen, Liu, He, and Du]{chen2023deformable}
Junyu Chen, Yihao Liu, Yufan He, and Yong Du.
\newblock Deformable cross-attention transformer for medical image registration.
\newblock In \emph{International Workshop on Machine Learning in Medical Imaging}, pages 115--125. Springer, 2023.

\bibitem[Chen et~al.(2021)Chen, Sun, Ellington, Xing, and Song]{chen2021multi}
Xinshi Chen, Haoran Sun, Caleb Ellington, Eric Xing, and Le~Song.
\newblock Multi-task learning of order-consistent causal graphs.
\newblock \emph{Advances in Neural Information Processing Systems}, 34:\penalty0 11083--11095, 2021.

\bibitem[Chickering(2002)]{chickering2002optimal}
David~Maxwell Chickering.
\newblock Optimal structure identification with greedy search.
\newblock \emph{Journal of machine learning research}, 3\penalty0 (Nov):\penalty0 507--554, 2002.

\bibitem[Evangelista et~al.(2023)Evangelista, Morotti, Piccolomini, and Nagy]{evangelista2023ambiguity}
Davide Evangelista, Elena Morotti, Elena~Loli Piccolomini, and James Nagy.
\newblock Ambiguity in solving imaging inverse problems with deep-learning-based operators.
\newblock \emph{Journal of Imaging}, 9\penalty0 (7):\penalty0 133, 2023.

\bibitem[Evangelista et~al.(2025)Evangelista, Loli~Piccolomini, Morotti, and Nagy]{evangelista2025tobe}
Davide Evangelista, Elena Loli~Piccolomini, Elena Morotti, and James~G Nagy.
\newblock To be or not to be stable, that is the question: understanding neural networks for inverse problems.
\newblock \emph{SIAM Journal on Scientific Computing}, 47\penalty0 (1):\penalty0 C77--C99, 2025.

\bibitem[Guo et~al.(2023)Guo, Cao, and Chen]{guo2023Transformer}
Ruchi Guo, Shuhao Cao, and Long Chen.
\newblock Transformer meets boundary value inverse problems.
\newblock In \emph{International Conference on Learning Representations}, 2023.
\newblock URL \url{https://openreview.net/forum?id=NzMMR26pSj}.

\bibitem[Hoyer et~al.(2008)Hoyer, Janzing, Mooij, Peters, and Sch{\"o}lkopf]{hoyer2008nonlinear}
Patrik Hoyer, Dominik Janzing, Joris~M Mooij, Jonas Peters, and Bernhard Sch{\"o}lkopf.
\newblock Nonlinear causal discovery with additive noise models.
\newblock \emph{Advances in neural information processing systems}, 21, 2008.

\bibitem[Khemakhem et~al.(2021)Khemakhem, Monti, Leech, and Hyvarinen]{khemakhem2021causal}
Ilyes Khemakhem, Ricardo Monti, Robert Leech, and Aapo Hyvarinen.
\newblock Causal autoregressive flows.
\newblock In \emph{International conference on artificial intelligence and statistics}, pages 3520--3528. PMLR, 2021.

\bibitem[Lachapelle et~al.(2019)Lachapelle, Brouillard, Deleu, and Lacoste-Julien]{lachapelle2019gradient}
S{\'e}bastien Lachapelle, Philippe Brouillard, Tristan Deleu, and Simon Lacoste-Julien.
\newblock Gradient-based neural dag learning.
\newblock \emph{arXiv preprint arXiv:1906.02226}, 2019.

\bibitem[Li et~al.(2020)Li, Xiao, and Tian]{li2020supervised}
Hebi Li, Qi~Xiao, and Jin Tian.
\newblock Supervised whole dag causal discovery.
\newblock \emph{arXiv preprint arXiv:2006.04697}, 2020.

\bibitem[Liu and Yu(2025)]{liu2025neural}
Ning Liu and Yue Yu.
\newblock Neural interpretable pdes: Harmonizing fourier insights with attention for scalable and interpretable physics discovery.
\newblock \emph{arXiv preprint arXiv:2505.23106}, 2025.

\bibitem[Lu et~al.(2021)Lu, Wu, Hern{\'a}ndez-Lobato, and Sch{\"o}lkopf]{lu2021invariant}
Chaochao Lu, Yuhuai Wu, Jos{\'e}~Miguel Hern{\'a}ndez-Lobato, and Bernhard Sch{\"o}lkopf.
\newblock Invariant causal representation learning for out-of-distribution generalization.
\newblock In \emph{International Conference on Learning Representations}, 2021.

\bibitem[Lu and Yu(2025)]{lu2025transformer}
Fei Lu and Yue Yu.
\newblock Transformer learns the cross-task prior and regularization for in-context learning.
\newblock \emph{arXiv preprint arXiv:2505.12138}, 2025.

\bibitem[Lu and Gao(2023)]{lu2023meta}
Songtao Lu and Tian Gao.
\newblock Meta-dag: Meta causal discovery via bilevel optimization.
\newblock In \emph{ICASSP 2023-2023 IEEE International Conference on Acoustics, Speech and Signal Processing (ICASSP)}, pages 1--5. IEEE, 2023.

\bibitem[Lu et~al.(2024)Lu, Letey, Zavatone-Veth, Maiti, and Pehlevan]{lu2024asymptotic}
Yue~M Lu, Mary~I Letey, Jacob~A Zavatone-Veth, Anindita Maiti, and Cengiz Pehlevan.
\newblock Asymptotic theory of in-context learning by linear attention.
\newblock \emph{arXiv preprint arXiv:2405.11751}, 2024.

\bibitem[Mooij et~al.(2009)Mooij, Janzing, Peters, and Sch{\"o}lkopf]{mooij2009regression}
Joris Mooij, Dominik Janzing, Jonas Peters, and Bernhard Sch{\"o}lkopf.
\newblock Regression by dependence minimization and its application to causal inference in additive noise models.
\newblock In \emph{Proceedings of the 26th annual international conference on machine learning}, pages 745--752, 2009.

\bibitem[Ng et~al.(2020)Ng, Ghassami, and Zhang]{ng2020role}
Ignavier Ng, AmirEmad Ghassami, and Kun Zhang.
\newblock On the role of sparsity and dag constraints for learning linear dags.
\newblock \emph{Advances in Neural Information Processing Systems}, 33:\penalty0 17943--17954, 2020.

\bibitem[Ovadia et~al.(2024)Ovadia, Kahana, Stinis, Turkel, Givoli, and Karniadakis]{ovadia2024vito}
Oded Ovadia, Adar Kahana, Panos Stinis, Eli Turkel, Dan Givoli, and George~Em Karniadakis.
\newblock Vito: Vision transformer-operator.
\newblock \emph{Computer Methods in Applied Mechanics and Engineering}, 428:\penalty0 117109, 2024.

\bibitem[Pearl et~al.(2000)]{pearl2000models}
Judea Pearl et~al.
\newblock Models, reasoning and inference.
\newblock \emph{Cambridge, UK: CambridgeUniversityPress}, 19\penalty0 (2), 2000.

\bibitem[Peters and B{\"u}hlmann(2014)]{peters2014identifiability}
Jonas Peters and Peter B{\"u}hlmann.
\newblock Identifiability of gaussian structural equation models with equal error variances.
\newblock \emph{Biometrika}, 101\penalty0 (1):\penalty0 219--228, 2014.

\bibitem[Peters et~al.(2012)Peters, Mooij, Janzing, and Sch{\"o}lkopf]{peters2012identifiability}
Jonas Peters, Joris Mooij, Dominik Janzing, and Bernhard Sch{\"o}lkopf.
\newblock Identifiability of causal graphs using functional models.
\newblock \emph{arXiv preprint arXiv:1202.3757}, 2012.

\bibitem[Sachs et~al.(2005)Sachs, Perez, Pe'er, Lauffenburger, and Nolan]{sachs2005causal}
Karen Sachs, Omar Perez, Dana Pe'er, Douglas~A Lauffenburger, and Garry~P Nolan.
\newblock Causal protein-signaling networks derived from multiparameter single-cell data.
\newblock \emph{Science}, 308\penalty0 (5721):\penalty0 523--529, 2005.

\bibitem[Shimizu et~al.(2006)Shimizu, Hoyer, Hyv{\"a}rinen, Kerminen, and Jordan]{shimizu2006linear}
Shohei Shimizu, Patrik~O Hoyer, Aapo Hyv{\"a}rinen, Antti Kerminen, and Michael Jordan.
\newblock A linear non-gaussian acyclic model for causal discovery.
\newblock \emph{Journal of Machine Learning Research}, 7\penalty0 (10), 2006.

\bibitem[Subbaswamy and Saria(2020)]{subbaswamy2020development}
Adarsh Subbaswamy and Suchi Saria.
\newblock From development to deployment: dataset shift, causality, and shift-stable models in health ai.
\newblock \emph{Biostatistics}, 21\penalty0 (2):\penalty0 345--352, 2020.

\bibitem[Vladymyrov et~al.(2024)Vladymyrov, Von~Oswald, Sandler, and Ge]{vladymyrov2024linear}
Max Vladymyrov, Johannes Von~Oswald, Mark Sandler, and Rong Ge.
\newblock Linear transformers are versatile in-context learners.
\newblock \emph{Advances in Neural Information Processing Systems}, 37:\penalty0 48784--48809, 2024.

\bibitem[Wan et~al.(2024)Wan, Wu, Hu, Chu, and Li]{wan2024bridging}
Guangya Wan, Yuqi Wu, Mengxuan Hu, Zhixuan Chu, and Sheng Li.
\newblock Bridging causal discovery and large language models: A comprehensive survey of integrative approaches and future directions.
\newblock \emph{arXiv preprint arXiv:2402.11068}, 2024.

\bibitem[Wu et~al.(2024)Wu, Kuang, Zhu, Wang, Zheng, Han, Li, Chen, Wu, and Zhang]{wu2024causality}
Anpeng Wu, Kun Kuang, Minqin Zhu, Yingrong Wang, Yujia Zheng, Kairong Han, Baohong Li, Guangyi Chen, Fei Wu, and Kun Zhang.
\newblock Causality for large language models.
\newblock \emph{arXiv preprint arXiv:2410.15319}, 2024.

\bibitem[Yin et~al.(2024)Yin, Gao, Yu, and Ji]{yin2024effective}
Naiyu Yin, Tian Gao, Yue Yu, and Qiang Ji.
\newblock Effective causal discovery under identifiable heteroscedastic noise model.
\newblock In \emph{Proceedings of the AAAI Conference on Artificial Intelligence}, volume~38, pages 16486--16494, 2024.

\bibitem[Yu et~al.(2019)Yu, Chen, Gao, and Yu]{yu2019dag}
Yue Yu, Jie Chen, Tian Gao, and Mo~Yu.
\newblock Dag-gnn: Dag structure learning with graph neural networks.
\newblock In \emph{International Conference on Machine Learning}, pages 7154--7163. PMLR, 2019.

\bibitem[Yu et~al.(2021)Yu, Gao, Yin, and Ji]{yu2021dags}
Yue Yu, Tian Gao, Naiyu Yin, and Qiang Ji.
\newblock Dags with no curl: An efficient dag structure learning approach.
\newblock In \emph{International Conference on Machine Learning}, pages 12156--12166. Pmlr, 2021.

\bibitem[Yu et~al.(2024)Yu, Liu, Lu, Gao, Jafarzadeh, and Silling]{yu2024nonlocal}
Yue Yu, Ning Liu, Fei Lu, Tian Gao, Siavash Jafarzadeh, and Stewart Silling.
\newblock Nonlocal attention operator: Materializing hidden knowledge towards interpretable physics discovery.
\newblock In \emph{Annual Conference on Neural Information Processing Systems}, 2024.

\bibitem[Zhang et~al.(2024{\natexlab{a}})Zhang, Jennings, Hilmkil, Pawlowski, Zhang, and Ma]{zhang2024towards}
Jiaqi Zhang, Joel Jennings, Agrin Hilmkil, Nick Pawlowski, Cheng Zhang, and Chao Ma.
\newblock Towards causal foundation model: on duality between optimal balancing and attention.
\newblock In \emph{Forty-first International Conference on Machine Learning}, 2024{\natexlab{a}}.

\bibitem[Zhang et~al.(2017)Zhang, Huang, Zhang, Glymour, and Sch{\"o}lkopf]{zhang2017causal}
Kun Zhang, Biwei Huang, Jiji Zhang, Clark Glymour, and Bernhard Sch{\"o}lkopf.
\newblock Causal discovery from nonstationary/heterogeneous data: Skeleton estimation and orientation determination.
\newblock In \emph{IJCAI: Proceedings of the Conference}, volume 2017, page 1347, 2017.

\bibitem[Zhang et~al.(2024{\natexlab{b}})Zhang, Frei, and Bartlett]{zhang2024trained}
Ruiqi Zhang, Spencer Frei, and Peter~L Bartlett.
\newblock Trained transformers learn linear models in-context.
\newblock \emph{Journal of Machine Learning Research}, 25\penalty0 (49):\penalty0 1--55, 2024{\natexlab{b}}.

\bibitem[Zhang et~al.(2022)Zhang, Ng, Gong, Liu, Abbasnejad, Gong, Zhang, and Shi]{zhang2022truncated}
Zhen Zhang, Ignavier Ng, Dong Gong, Yuhang Liu, Ehsan Abbasnejad, Mingming Gong, Kun Zhang, and Javen~Qinfeng Shi.
\newblock Truncated matrix power iteration for differentiable dag learning.
\newblock \emph{Advances in Neural Information Processing Systems}, 35:\penalty0 18390--18402, 2022.

\bibitem[Zheng et~al.(2018)Zheng, Aragam, Ravikumar, and Xing]{zheng2018dags}
Xun Zheng, Bryon Aragam, Pradeep~K Ravikumar, and Eric~P Xing.
\newblock Dags with no tears: Continuous optimization for structure learning.
\newblock \emph{Advances in neural information processing systems}, 31, 2018.

\bibitem[Zheng et~al.(2020)Zheng, Dan, Aragam, Ravikumar, and Xing]{zheng2020learning}
Xun Zheng, Chen Dan, Bryon Aragam, Pradeep Ravikumar, and Eric Xing.
\newblock Learning sparse nonparametric dags.
\newblock In \emph{International Conference on Artificial Intelligence and Statistics}, pages 3414--3425. PMLR, 2020.

\bibitem[Zhou et~al.(2022)Zhou, He, and Ni]{zhou2022causal}
Fangting Zhou, Kejun He, and Yang Ni.
\newblock Causal discovery with heterogeneous observational data.
\newblock In \emph{Uncertainty in Artificial Intelligence}, pages 2383--2393. PMLR, 2022.

\end{thebibliography}






\newpage 
\appendix
\section{Implementation Details}\label{app:implementation_details}
In this section, we discuss our implementation in terms of three aspects: attention mechanism model, training procedure, and the augmented Lagrangian optimization.

\textbf{Attention Mechanism Model. }We implement our nonlinear kernel map between data observations and weighted adjacency matrices using linear transformers. The key hyperparameters include the number of attention heads $r$, the number of transformer layers $nb$, and the dimension $k$ for the parameters $\Wb_{1:L}^Q$ and $\Wb_{1:L}^K$. These parameters are chosen to ensure that the kernel map is expressive enough to generalize to unseen data observations. Specifically, we set $r = 1$ across all settings. For the number of layers, we use $nb = 15$ when $d = 5$ or $d = 10$, and $nb = 20$ when $d = 20$. The dimension $k$ is used to reduce the input observation size $n$, and we typically set $k = \sqrt{n}$. Accordingly, we choose $k = 10$ for $n = 100$ and $n = 50$, and $k = 5$ for $n = 25$.

\textbf{Augmented Lagrangian Optimization. }We initialize the Lagrangian multipliers with $\alpha = 0$ and $\rho = 1$. The progress rate is set to $c = \frac{1}{4}$, and the convergence tolerance is $\epsilon = 10^{-5}$. For each value of $\alpha$, we evaluate the acyclicity constraint $\sum_{\eta=1}^M h(\hat{A}^\eta)$. If the constraint does not decrease by a factor of $c$ (i.e., is not reduced to $\frac{1}{4}$ of its previous value), we increase $\rho$ by a factor of 10 and repeat the optimization. If the reduction criterion is met, we update the multiplier as $\alpha \leftarrow \alpha + \rho \sum_{\eta=1}^M h(\hat{A}^\eta)$. The optimization terminates once the constraint value satisfies $\sum_{\eta=1}^M h(\hat{A}^\eta) < \epsilon$.

\textbf{Training Procedures.} We use the Adam optimizer across all settings with a fixed batch size of 100. When $\alpha = 0$ and $\rho = 1$, we train for 5000 epochs with an initial learning rate of $3 \times 10^{-4}$. The learning rate decays by a factor of 0.7 every 1000 steps. For subsequent values of the Lagrangian multiplier, we reduce the number of training epochs to 100 and set the learning rate to $1 \times 10^{-4}$.

\section{Data Visualization}\label{app:data_visualization}
We describe the data generation process for both heterogeneous and order-consistent settings. Since both types use the same linear SEM with additive noise to generate observations from ground-truth weighted adjacency matrices, the primary difference lies in the structure of these matrices. Therefore, we illustrate the possible sets of ground-truth adjacency matrices $A^{gt}$ for each setting in Figure~\ref{fig:Data_visual}. As shown in Figure~\ref{fig:Data_visual}(a), the ground-truth weighted adjacency matrices for heterogeneous data share the same DAG structure but differ in their edge weights. According to Figure~\ref{fig:Data_visual}(b), the weighted adjacency matrices for order-consistent data vary in structure but all respect the same underlying topological order.

\begin{figure}[hpt]
    \centering
    \includegraphics[width=1\linewidth]{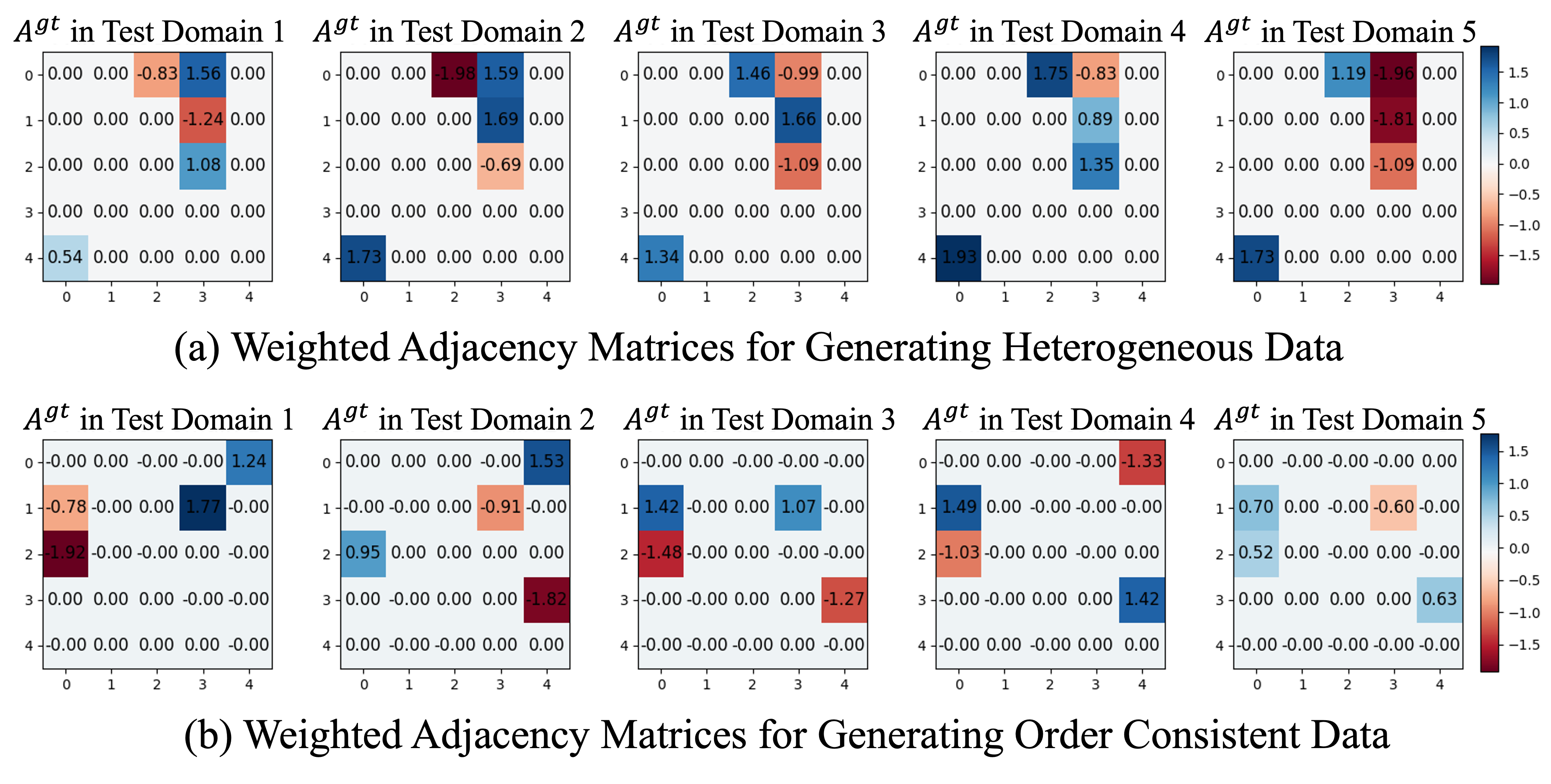}
    \caption{Visualization of the Ground Truth Weighted Adjacency Matrices for Heterogeneous Data and Order Consist Data.}
    \label{fig:Data_visual}
\end{figure}
\section{Detailed Empirical Results}\label{app:detailed_empirical_results}
To provide a comprehensive evaluation of our proposed ADAG method, we also test its performance on the real-world Sachs dataset, as presented in Section~\ref{app:real_data_results}. Additionally, we conduct an ablation study in Section~\ref{app:ablation_M} to examine how the number of training domains influences the generalization ability of the pre-trained model. This study offers empirical insights into the amount of data required for effective pre-training. Furthermore, in Section~\ref{app:various_noise}, we demonstrate that ADAG remains effective under linear SEMs with non-Gaussian noise.
\subsection{Real Data Results}\label{app:real_data_results}
Our ADAG method follows a framework consisting of a pre-training phase and a zero-shot inference procedure for DAG learning. The pre-training phase relies on access to sufficient multi-domain training data. As a result, ADAG cannot be directly applied to the Sachs dataset, which contains only 853 observations across 11 variables. To address this limitation, we design a protocol for evaluating ADAG on real data. This protocol includes two components: (1) pre-training on synthetic multi-domain data, and (2) evaluation on a hold-out subset of the Sachs dataset with limited observations.

\textbf{Pre-train on Synthetic Data. }Due to the limited availability of abundant training data, we propose generating synthetic multi-domain data with a controlled use of ground-truth information. Based on the widely accepted causal structure of the 11 variables in the Sachs dataset, we infer a topological ordering of the variables: (Plcg, PKC, PIP3, PKA, PIP2, Raf, P38, Jnk, Mek, Erk, Akt). Using this ordering, we generate 100,000 distinct DAGs that differ in structure but all conform to the specified topological order. For each DAG, we generate its corresponding weighted adjacency matrix and $n$ data observations. This results in $M = 100{,}000$ training domains, each containing $n$ observations, which we use to pre-train the nonlinear kernel map. It is worth noting that our proposed protocol is just one of several possible approaches for simulating multi-domain Sachs data. Its main advantage is that it provides access to the ground-truth DAGs and weighted adjacency matrices, offering clear insight into how well the model is trained. Moreover, it enables evaluation of the model’s generalization ability by generating additional synthetic test domains. However, since the true parameter distributions are unknown, the synthetic parameters may differ from those of the real data. An alternative approach would be to shuffle and permute the original 853 samples to simulate different domains. While this yields more realistic data, potentially closer to the unknown true parameters, it lacks explicit ground-truth information, making it difficult to assess model training quality or determine whether sub-optimization problems have converged. This can ultimately hinder the effectiveness of pretraining the nonlinear kernel map.

\textbf{Hold-out Evaluation with Limited Observations.  }Based on its performance on synthetic data presented in the main paper, our ADAG method is particularly robust in low-sample regimes. To better highlight its effectiveness in such scenarios and to avoid the computational cost of training on the full dataset, we propose sampling a subset of $n = 100$ observations from the original 853 samples. We then evaluate our ADAG method against state-of-the-art baselines using this held-out subset for comparison.

\begin{table}[hpt]
    \centering
    \caption{Empirical Results on Real Data Sachs}
    \label{tab:real_sachs}
    \begin{tabular}{cccc}
    \toprule
    Methods  & $n$ & SHD($\downarrow$) & $\#$ estimated edges  \\
    \hline 
    \multirow{2}{*}{NOTEARS}  &  853  & 13 & 17 \\
    & 100 & 19 & 17 \\
    \hline 
    \multirow{2}{*}{DAGMA} & 853 & 14 & 6  \\
    & 100 & 15 & 13 \\
    \hline 
    ADAG & 100 & 13 & 17 \\
    \bottomrule
    \end{tabular}
\end{table}

We present the empirical results of our ADAG method alongside state-of-the-art single-task DAG learning methods, NOTEARS and DAGMA, in Table~\ref{tab:real_sachs}. DAG learning accuracy is evaluated using Structural Hamming Distance (SHD), and we also report the number of predicted edges. Hyperparameters, including sparsity constraint coefficients and thresholds, are extensively tuned to optimize SHD performance. We first evaluate NOTEARS and DAGMA on the full set of 853 observations, where both methods perform well. In particular, NOTEARS achieves an SHD of 13 with 17 predicted edges, matching state-of-the-art performance. However, when the number of observations is limited to 100, their performance degrades, with SHDs of 19 (NOTEARS) and 15 (DAGMA). In contrast, our ADAG method achieves an SHD of 13 on the same 100-observation subset, matching the best performance of NOTEARS on the full dataset. This demonstrates ADAG’s robustness in low-sample regimes and its ability to generalize effectively from pre-training.

As mentioned in the \textbf{Pre-training on Synthetic Data} section, having more accurate prior knowledge of the parameter distributions allows our ADAG method to learn a better pre-trained model, which can potentially lead to improved performance on real data.

\subsection{Ablation on Number of Domains}\label{app:ablation_M}
During the pre-training phase, we observe that a sufficiently large number of training domains is necessary to effectively train the nonlinear kernel map, enabling it to produce accurate weighted adjacency matrix predictions for unseen data observations. Hence, we perform an ablation study which varies the number of data domains for training and evaluates the pre-trained models on $1000$ test domains data. We set the number of training domains to $M = 0, 500, 1000, 5000, 10000, 15000, 20000, 30000, 40000, 50000, 60000, 70000$, and report the performance in terms of reconstruction loss values and relative errors on the test domains. 
\begin{figure}[hpt]
    \centering
    \includegraphics[width=1\linewidth]{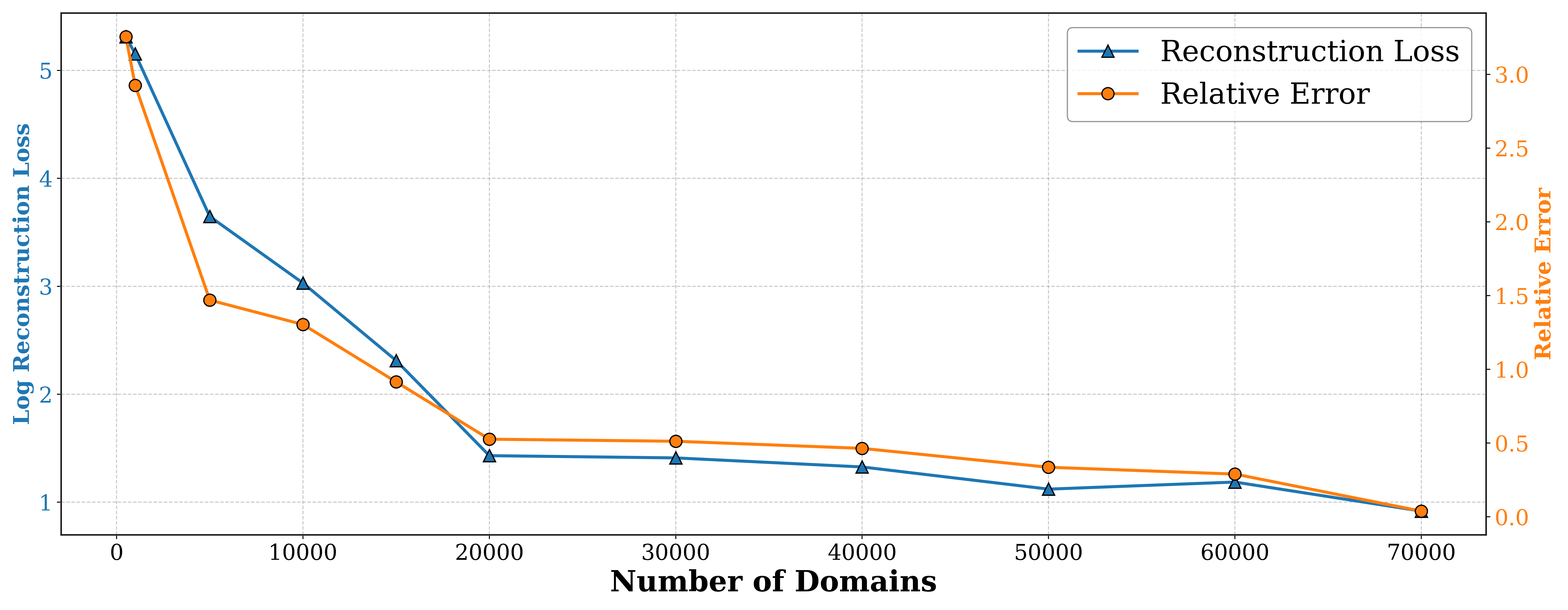}
    \caption{Nonlinear Kernel Map Generalization vs Number of Domains}
    \label{fig:M_domains}
\end{figure}
Figure~\ref{fig:M_domains} shows that as the number of training domains increases, both the reconstructed input data observations and the estimated weighted adjacency matrices become closer to the ground truth. When $M = 70{,}000$, the reconstruction loss ($2.5030$) and relative error ($0.0386$) on the test domains are comparable to those on the training domains (reconstruction loss: $2.4980$, relative error: $0.0242$). Hence, we use $M\geq 70000$ domains for training across all settings.

\subsection{Ablation Study on Various Types of Noise}\label{app:various_noise}
We also perform ablation study to apply our ADAG on data generated from linear SEMs but with non-Gaussian noise. As shown in Table \ref{tab:noise_types}, we generate synthetic data with exponential and Gumbel noise distributions, and compare our ADAG method against the state-of-the-art single-task and multi-task DAG learning baselines. 
\begin{table}[hpt]
    \centering
    \caption{Empirical Results of Baselines and our ADAG approach on Linear SEM Data Generated from ER1 d5 Graphs with Various Types of Noise.}
    \label{tab:noise_types}
    \scalebox{0.92}{
    \begin{tabular}{cccccc}
    \toprule
     Additive Noise Types & Methods & SHD$\downarrow$ & $\#$ edges & Relative Error $\downarrow$ & Runtime(s) $\downarrow$ \\
    \hline 
    \multirow{6}{*}{Exponential} & NOTEARS  &  0.5640\scalebox{0.7}{$\pm 0.0238$} & 5.3480\scalebox{0.7}{$\pm 0.0243$} & 0.1620\scalebox{0.7}{$\pm 0.0046$} &  0.0982 \\
    & DAGMA & 0.5590\scalebox{0.7}{$\pm 0.0152$} & 5.3900\scalebox{0.7}{$\pm 0.0147$} & 0.1570\scalebox{0.7}{$\pm 0.0044$} &  0.5414 \\
    & MetaDAG & 5.5000\scalebox{0.7}{$\pm 0.5000$} & 3.0000\scalebox{0.7}{$\pm 1.0000$} & 0.9489\scalebox{0.7}{$\pm 0.0599$} &  0.2801 \\
    & CD-NOD & 1.1500\scalebox{0.7}{$\pm 0.1212$} & - & - &  2.8525 \\
    & MultiDAG & 1.3700\scalebox{0.7}{$\pm 0.0363$} & 4.9940\scalebox{0.7}{$\pm 0.1301$} & 0.2755\scalebox{0.7}{$\pm 0.0085$} &  0.0310 \\
    & ADAG & \bf{0.1130}\scalebox{0.7}{$\pm \bf{0.0150}$} & 4.8970\scalebox{0.7}{$\pm 0.0161$} & 0.2017\scalebox{0.7}{$\pm 0.0080$} &  \bf{0.0003} \\
    \midrule
    \multirow{6}{*}{Gumbel} & NOTEARS  &  0.4599\scalebox{0.7}{$\pm 0.0572$} & 5.2969\scalebox{0.7}{$\pm 0.0405$} & 0.1460\scalebox{0.7}{$\pm 0.0071$} &  0.1002 \\
    & DAGMA & 0.4639\scalebox{0.7}{$\pm 0.0575$} & 5.3330\scalebox{0.7}{$\pm 0.0348$} & \bf{0.1428}\scalebox{0.7}{$\pm 0.0053$} &  0.5696 \\
    & MetaDAG & 4.5000\scalebox{0.7}{$\pm 0.5000$} & 4.0000\scalebox{0.7}{$\pm 1.5000$} & 0.9132\scalebox{0.7}{$\pm 0.0470$} &  0.2488 \\
    & CD-NOD & 1.2513\scalebox{0.7}{$\pm 0.2029$} & - & - &  2.6850 \\
    & MultiDAG & 1.2261\scalebox{0.7}{$\pm 0.0442$} & 6.0031\scalebox{0.7}{$\pm 0.0500$} & 0.2996\scalebox{0.7}{$\pm 0.0007$} &  0.0278 \\
    & ADAG & \bf{0.0570}\scalebox{0.7}{$\pm \bf{0.0050}$} & 4.9450\scalebox{0.7}{$\pm \bf{0.0057}$} & 0.1442\scalebox{0.7}{$\pm 0.0021$} &  \bf{0.0003} \\
    \bottomrule
    \end{tabular}
}    
\end{table}
The empirical results are consistent with those reported in the main paper for data with equal-variance Gaussian noise. Our ADAG method achieves optimal performance in terms of DAG inference accuracy (lowest SHD) and zero-shot inference efficiency. Additionally, we report the standard deviation of the expected performance over 1,000 domains across three trials. Compared to all baselines, our method exhibits the smallest standard deviation in both SHD, highlighting the stability and reliability of the trained kernel map.



\section{Theoretical Justifications}\label{sec:theory}
Intuitively, the optimization problem in our pre-training process can be separated into two sub-problems: (i) learning the estimated adjacency matrix $A$ for each domain from input data $\Xb$ by minimizing the reconstruction loss $\| \Xb - A^T \Xb \|_F^2$ under the acyclicity constraint, which corresponds to the standard DAG learning problem; and (ii) learning the nonlinear maps from the input data $\Xb$ of each domain to its corresponding weighted adjacency matrix $A$. A well-trained ADAG model requires both sub-problems to be effectively solved.

In the following section, we first discuss whether the weighted adjacency matrices with the ground-truth DAG structure can be identified for all domains (Section~\ref{app:graph_ident_theory}), and then examine the identifiability of the parameters in the nonlinear kernel map (Section~\ref{app:params_ident_theory}).

\subsection{Graph Identifiability}\label{app:graph_ident_theory}
During the pre-training phase of our proposed ADAG approach, we inherently perform multi-task DAG learning across $M$ domains.

When all $M$ domains have sufficient sample complexity, the identifiability problem reduces to whether the causal graph for each individual domain can be uniquely identified from its corresponding observations. For the linear SEM adopted in our framework, existing identifiability results show that the causal graph is identifiable under the following conditions: (1) the additive noise is non-Gaussian~\citep{shimizu2006linear}, or (2) the additive noise is Gaussian with equal noise variances~\citep{peters2014identifiability}. Based on these results, we assert that if the SEMs are linear and the noise satisfies either of these conditions, our method can identify the unique causal graph for each domain. These identifiability results may also be extended to nonlinear SEMs with additive noise, as discussed in \citet{hoyer2008nonlinear}, \citet{mooij2009regression}, and \citet{peters2012identifiability}.

A more interesting scenario occurs when the observations from some domains are not sufficiently complex to identify a unique DAG. In \cite{chen2021multi}, it was shown that by minimizing the joint data loss from all $M$ domains as discussed in our Section 3.3, this setting is able to recover the order of non-identifiable graphs if (1) the sample complexity index $\dfrac{d}{s}\sqrt{\dfrac{n}{d\log d}\dfrac{(M')^2}{M}}$ is sufficiently large, (2) the sample size $n$ is also sufficiently large (on the order of $\log M + (p+1)\log d$), and (3) the total domain number $M$ is bounded above by $O(d\log d)$. Here, $d$ is the number of random variables, $n$ is the number of observations in each domain, $p$ is the maximum number of parents in DAGs, $s$ is the size of the support union, and $M'$ is the number of domains with identifiable data among the total $M$ domains. While we anticipate the same identifiability results hold true for our learning problem, we also point out that our ADAG focuses on the small data regime, i.e., $n$ is of a similar size as $d$. Under this circumstance, conditions (1) and (2) may be violated. It suggests a possible relaxation of the theoretical results in \cite{chen2021multi} and an improved identifiability property under our foundation model setting. We leave such theoretical investigations to a future work.


\subsection{Parameter Identifiability of $A$}\label{app:params_ident_theory}


In addition to the capability of identifying the common topological ordering across all domains, ADAG is also capable of identifying the weighted adjacency matrix parameters, i.e., $A$. Under this setting, the learning of parameters can be seen as a discrete version of the learning problem considered in \cite{yu2024nonlocal}, and one can show that the space in which the values of $A$ are identifiable is the closure of a data-adaptive reproducing kernel Hilbert space (RKHS). In particular, when the common topological order is determined as a permutation $\pi$ over $[1:d]\coloneqq (1, 2, \cdots, d)$ over all domains, we denote the corresponding connectivity matrix as:
$$[C(\pi)]_{ij}=1, \text{ if }\pi(i)<\pi(j),$$
$$[C(\pi)]_{ij}=0, \text{ if }\pi(i)\geq\pi(j).$$
Then, we can rewrite the weighted adjacency matrix $A$ as:
$$A=\tilde{A}\circ C(\pi),$$
where $\tilde{A}_{ij}=0$ if $[C(\pi)]_{ij}=0$, and $\tilde{A}_{ij}=A_{ij}$ if $[C(\pi)]_{ij}=1$. $\circ$ denotes the Hadamard product. One can see that the parameter identifiability problem is equivalent to a learning problem of the $d(d-1)/2$ parameters in $\tilde{A}$. Without loss of generality, we consider $\pi$ to be the identity permutation to simplify the notations. Then, we have the following result:
\begin{lemma}[Space of Identifiability]
The loss function
\begin{equation}
\sum_{\eta=1}^M ||X^\eta_{1:d}(1:n)-(\tilde{A} \circ C(\pi))^\top X^\eta_{1:d}(1:n)||^2_F
\end{equation}
has a unique minimizer $\tilde{A}$ is the closure of a data-adaptive RKHS $H_G$ with a reproducing kernel $\bar{G}$ determined by the training data:
$$\bar{G}_{ijk}=[\rho'_{j}\rho'_{k}]^{-1}G_{jk}, \text{ if }\pi(i)<\pi(j),\, \pi(i)<\pi(k), \text{ else }\bar{G}_{ijk}=0.$$
Here $\rho'$ is the density of the empirical measure $\rho$ defined by
$$\rho'_{j}\coloneqq\dfrac{1}{Z}\sum_{\eta=1}^M\sum_{s=1}^n  |X^\eta_{j}(s)|,$$
with $Z$ being the normalizing constant, and $G$ is defined by
$${G}_{jk}\coloneqq\sum_{\eta=1}^M\sum_{s=1}^n  X^\eta_{j}(s) X^\eta_{k}(s).$$
\end{lemma}
\begin{proof}
The loss function can be expanded as:
\begin{align*}
&\sum_{\eta=1}^M ||X^\eta_{1:d}(1:n)-(\tilde{A} \circ C(\pi))^\top X^\eta_{1:d}(1:n)||^2_F\\
=&\sum_{\eta=1}^M\sum_{s=1}^n ||(\tilde{A} \circ C(\pi))^\top X^\eta_{1:d}(s)||^2-2\sum_{\eta=1}^M\sum_{s=1}^n (X^\eta_{1:d}(s))^\top(\tilde{A}^\eta \circ C(\pi))^\top X^\eta_{1:d}(s)+Const\\
=&\langle \mathcal{L}_{\bar{G}} \tilde{A},\tilde{A} \rangle_{L^2_{\rho}}-2\langle \tilde{A},(\tilde{A})^D \rangle_{L^2_{\rho}} +Const.
\end{align*}
$\mathcal{L}_{\bar{G}}$ is an operator mapping from an upper triangular $d\times d$ matrix to another upper triangular $d\times d$ matrix, defined as: 
$$(\mathcal{L}_{\bar{G}} \tilde{A})_{ij}= \sum_{\eta=1}^M\sum_{s=1}^n \sum_{k=i+1}^d \tilde{A}_{ik} X^\eta_{j}(s) X^\eta_{k}(s)= \sum_{k=i+1}^d \tilde{A}_{ik} \bar{G}_{ijk},$$
and $(\tilde{A})^D$ is an upper triangular $d\times d$ matrix satisfying:
$$\langle \tilde{A},(\tilde{A})^D \rangle_{L^2_{\rho}}=\sum_{\eta=1}^M\sum_{s=1}^n (X^\eta_{1:d}(s))^\top(\tilde{A} \circ C(\pi))^\top X^\eta_{1:d}(s).$$
This loss function has a unique minimizer in Null$(\mathcal{L}_{\bar{G}})^\perp$.
\end{proof}

Intuitively, increasing $M$ and $n$ helps to include more data, which would enhance the invertibility of $\bar{G}$ and enlarge the space of identifiability for $\tilde{A}$. For further discussions on how linear transformer functions enhance the identifiability and solve the inverse linear regression problem, we refer to \cite{yu2024nonlocal} and \cite{lu2025transformer}.

\end{document}